\documentclass[11pt]{article}

\usepackage[final]{acl}
\usepackage{times}
\usepackage{latexsym}
\usepackage{hyperref}
\usepackage[T1]{fontenc}
\usepackage[utf8]{inputenc}
\usepackage{microtype}
\usepackage{inconsolata}
\usepackage{listings}
\usepackage{xcolor}
\usepackage{graphicx}
\usepackage{amssymb}
\usepackage{mathtools}
\usepackage{mathrsfs}
\usepackage{amsmath}
\usepackage{amssymb}
\usepackage{amsthm}
\usepackage{float}

\newtheorem{lemma}{Lemma}
\newtheorem{theorem}{Theorem}

\usepackage{booktabs}
\usepackage[most]{tcolorbox}
\usepackage{multirow}
\title{Navigating Noisy Feedback: Enhancing Reinforcement Learning with Error-Prone Language Models}

\author{
 \textbf{Muhan Lin\textsuperscript{1}},
 \textbf{Shuyang Shi\textsuperscript{1}},
 \textbf{Yue Guo\textsuperscript{1}},
 \textbf{Behdad Chalaki\textsuperscript{2}},
\\
 \textbf{Vaishnav Tadiparthi\textsuperscript{2}},
 \textbf{Ehsan Moradi Pari\textsuperscript{2}},
 \textbf{Simon Stepputtis\textsuperscript{1}},
 \textbf{Joseph Campbell \textsuperscript{1}},
\\
 \textbf{Katia Sycara\textsuperscript{1}}
\\
\\
 \textsuperscript{1}Carnegie Mellon University,
 \textsuperscript{2}Honda Research Institute USA
\\
 \small{
\textbf{Correspondence:} \href{mailto:email@domain}{muhanlin@cs.cmu.edu}
 }
}

\begin{document}
\maketitle
\begin{abstract}
The correct specification of reward models is a well-known challenge in reinforcement learning.
Hand-crafted reward functions often lead to inefficient or suboptimal policies and may not be aligned with user values.
Reinforcement learning from human feedback is a successful technique that can mitigate such issues, however, the collection of human feedback can be laborious.
Recent works have solicited feedback from pre-trained large language models rather than humans to reduce or eliminate human effort, however, these approaches yield poor performance in the presence of hallucination and other errors.
This paper studies the advantages and limitations of reinforcement learning from large language model feedback and proposes a simple yet effective method for soliciting and applying feedback as a potential-based shaping function.
We theoretically show that inconsistent rankings – which approximate ranking errors – lead to uninformative rewards with our approach. Our method empirically improves convergence speed and policy returns over commonly used baselines even with significant ranking errors, and eliminates the need for complex post-processing of reward functions.
\end{abstract}

\section{Introduction}

The correct specification of task rewards is a well-known challenge in reinforcement learning (RL)~\cite{leike2018scalable}.
Complex tasks often necessitate complex, nuanced reward models, particularly as shaping terms may be required to guide exploration.
However, hand-crafting these reward functions is difficult and often leads to a phenomenon known as \textit{reward hacking}, wherein an agent learns to exploit a reward function for increased returns while yielding unexpected or undesired behavior~\citep{skalse2022defining}.
Reward hacking is symptomatic of the broader challenge of \textit{value alignment}, in which it is difficult for a human domain expert to fully and accurately specify the desired behavior of the learned policy in terms of a reward function.

In this study, we aim to eliminate the dependence on handcrafted reward functions by training agents with reward functions derived from data.
A notable method in this domain is reinforcement learning from human feedback (RLHF), where policy trajectories are ranked by humans.
These rankings are then used to learn a reward function which guides model training and facilitates value alignment.
This process is extremely costly in terms of human effort, however, requiring a significant number of rankings to train accurate reward models~\cite{casper2023open}. 

We can avoid the need for humans-in-the-loop by instead generating rankings with pre-trained large language models (LLMs) in a process known as reinforcement learning with AI feedback (RLAIF)~\citep{lee2023rlaif, bai2022constitutional, kim2023aligning}.
However, LLMs are well known to hallucinate and present false information as fact~\citep{zhang2023siren}, which reduces the accuracy and reliability of the resulting rankings.
This is often overcome through complex reward post-processing techniques, which may be task-specific and difficult to tune~\cite{klissarov2023motif}.

In this work, we propose a simple and effective strategy for reinforcement learning in the face of unreliable LLM feedback.
The core idea underlying our approach is to issue uninformative rewards for states in which the LLM is uncertain.
Thus, we avoid issuing potentially misleading rewards which allows us to train performant policies even in the face of significant ranking errors.
Building off the insight that certainty in language models is expressed through output consistency~\citep{tanneru2024quantifying}, we show that rewards issued from a potential-based scoring function learned over repeated rankings naturally reflect an LLM's uncertainty.

Our contributions are as follow, we 1) introduce a methodology for incorporating noisy LLM feedback into RL which out-performs prior SOTA; and 2) provide theoretical and empirical analysis showing that uncertain LLM outputs -- as given by inconsistent responses -- lead to uninformative rewards which improve convergence speed and policy returns in experiments. The codes of this work can be accessed \href{https://github.com/sy-shi/RLAIF\_ScoreDiff.git}{here}.

\begin{figure*}
     \centering
     \begin{minipage}[b]{0.61\textwidth}
         \centering
         \vspace{-10pt}
         \includegraphics[width=1\linewidth]{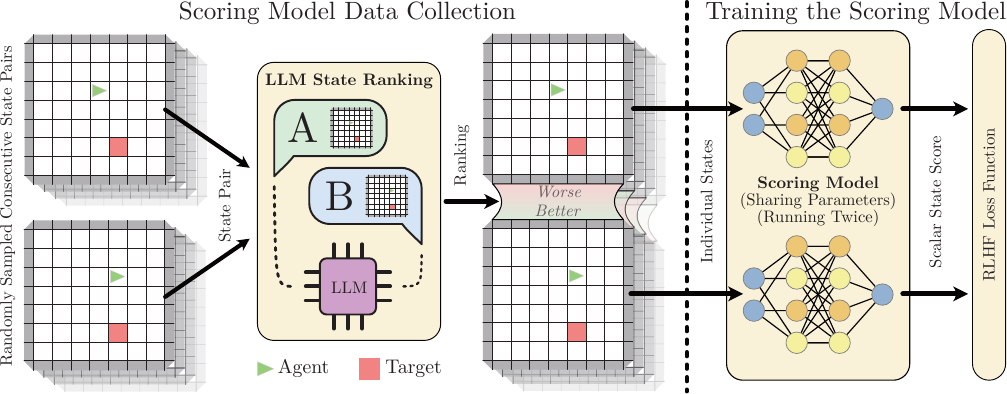}
         \caption*{(a) Training the scoring model from LLM-ranked state pairs.}
         \label{fig:overview-train-scoring}
     \end{minipage}
     \hfill
     \begin{minipage}[b]{0.37\textwidth}
         \centering
         \vspace{-10pt}
         \includegraphics[width=1\linewidth]{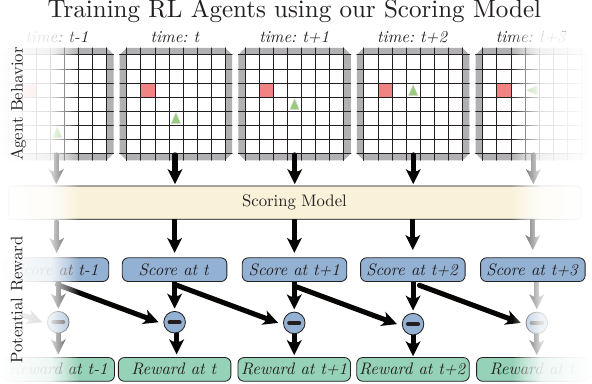}
         \caption*{(b) Agent training using the scoring model.}
         \label{fig:overview-train-agent}
     \end{minipage}
     \caption{Our approach: (a) We train our scoring model with randomly sampled \textit{consecutive} state pairs, which are ranked by an LLM with respect to task completion. The resulting dataset of ranked state pairs is utilized in an RLHF fashion to train a single scoring model, capable of providing a score for any novel state. (b) Using the scoring model, an RL agent is trained by scoring each state. Prior work uses this score as a reward; however, our approach utilizes the score differences as a potential reward.}
\end{figure*}

\vspace{-10pt}

\section{Related Works}

Constructing rewards based on human feedback has a long history \citep{thomaz2006reinforcement}. To efficiently use human domain knowledge and provide more generalizable rewards, human preference on episode segments \citep{sadigh2017active, christiano2017deep, biyik2019asking} and human demonstrations \citep{biyik2022learning} are distilled into models which serve as reward functions for RL. The method has witnessed great success in complex domains where rewards are difficult to specify such as training large language models (LLM) to align with human logic and common sense \citep{ziegler2019fine,ouyang2022training}.

One major drawback of RLHF is its requirement of exhaustive human participation to provide demonstrations and feedback. 
LLMs have shown deductive logic abilities comparable to humans in recent years \citep{du2023guiding}, and are able to substitute humans in reward issuing \citep{kwon2023reward,yu2023language,lee2023rlaif, xie2023text2reward}, or data collection and labeling for reward model training \citep{lee2023rlaif, klissarov2023motif}. 
While the former suffers from time and resource costs for training RL agents, the latter is becoming promising for training complex RL tasks \citep{wang2024rl}.

An outstanding challenge with leveraging LLM-based feedback is that the performance of RLHF is dependent on the quality of feedback received~\cite{casper2023open}.
Different LLMs have distinct probabilities of giving wrong feedback, thus leading to rewards of varying quality.
\citet{casper2023open} also suggests that comparison-based feedback may not be efficient and adequate to train reward models with noisy LLM outputs.
In this work, we analyze the training performance of reinforcement learning agents across various LLMs, each of which produce different error distributions in feedback.

Another challenge is that of the reward formulation itself. Many works train a model distilling LLM or human preference and use it as the reward model \citep{christiano2017deep, wang2024rl, klissarov2023motif, lee2023rlaif}, but in practice, this needs post-processing on outputs of the reward model such as filtering \citep{klissarov2023motif}, and normalization \citep{christiano2017deep}.
Our work posits that a reward function trained without complex post-processing and environment rewards would be more general and adaptable to various practical scenarios.

\vspace{-10pt}

\section{Background}
\textbf{Reinforcement Learning:} 
In reinforcement learning, an agent interacts with an environment and receives a reward for its current action at each time step, learning an optimal action policy to maximize the rewards over time. This procedure can be formulated as an infinite horizon discounted Markov Decision Process (MDP) \citep{sutton2018reinforcement}.

At each discrete timestep $t$ in this process, the agent observes environment state $s_t$ and takes action $a_t$, leading to the next environment state $s_{t+1}$ and a reward $r_t$.  An MDP is represented as a tuple $(\mathcal{S}, \mathcal{A}, \mathcal{R}, \mathcal{T}, \gamma)$, where $\mathcal{S}$ is a set of states, $\mathcal{A}$ is a set of actions, $\mathcal{R}: \mathcal{S} \mapsto \mathbb{R}$ is a reward function, $\mathcal{T}(s,a,s') = P(s' | s,a)$ is a transition function, and $\gamma$ is a discount factor. A stochastic policy $\pi(a|s): \mathcal{A} \times \mathcal{S} \mapsto [0, 1]$ indicates the probability that the agent selects action $a$ given the state $s$. The agent's goal is to learn $\pi$ maximizing the expected discounted return through training, given an initial state distribution.

\noindent\textbf{Preference-based Reinforcement Learning:} Our work is based on the framework of preference-based reinforcement learning, where the reward function is learned from preference labels over agent behaviors \citep{christiano2017deep, ibarz2018reward, lee2021pebble, lee2021b}. For a pair of states $(s_a, s_b)$, an annotator gives a preference label $y$ that indicates which state is ranked higher, considering which state is closer to the given task goal. The preference label $y \in \{0,1\}$, where 0 indicates $s_a$ is ranked higher than $s_b$, and 1 indicates $s_b$ is ranked higher than $s_a$. Given a parameterized state-score function $\sigma_\psi$, which is commonly called the potential function and usually equated with the reward model $r_\psi$, we compute the preference probability of a state pair with the standard Bradley-Terry model \citep{bradley1952rank},
\begin{equation}
\begin{aligned}
P_\psi[s_b \succ s_a] &= \frac{\exp \left(\sigma_\psi(s_b) \right)}{\exp \left(\sigma_\psi(s_a) \right) + \exp \left(\sigma_\psi(s_b) \right)}\\
&= sigmoid(\sigma_\psi(s_b) - \sigma_\psi(s_a)),
\end{aligned}
\end{equation}
where $s_b \succ s_a$ indicates $s_b$ is ranked higher than the state $s_a$. With a preference dataset $D={(s_a^i, s_b^i, y^i)}$, preference-based RL learns the state-score model $\sigma_\psi$ by minimizing the cross-entropy loss, which aims to maximize the score difference between the high and low states:

\begin{equation}
\begin{aligned}
\label{eq:loss}
    \mathcal{L} &= -\mathbb{E}_{(s_a, s_b, y) \sim \mathcal{D}} \bigg[ \mathbb{I}\{y = (s_a \succ s_b)\} \\ &\log P_\psi[s_a \succ s_b] + \mathbb{I}\{y = (s_b \succ s_a)\} \\& \log P_\psi[s_b \succ s_a]\bigg].
\end{aligned}
\end{equation}

This framework is used in both RLHF and RLAIF where rewards are issued directly from the state-score model and differ only in the choice of annotator, i.e., human or LLM.

\section{Methodology}
Despite the success of LLMs in few-shot task generalization,
these models are imperfect and yield sub-optimal performance in many areas.
One notable issue is the well-documented tendency of LLMs to hallucinate, which results in LLM-generated preference rankings frequently containing errors (see Table~\ref{tab:llm variants}).
These errors present major challenges for reinforcement learning from LLM feedback, as they result in noise in the learned score function.
Under the standard RLHF formulation where rewards are directly issued from the score function \citep{christiano2017deep}, this can lead to inefficient exploration at best and, at worst, trap the agent in sub-optimal local minima.
\subsection{Quantifying Feedback Error through Output Consistency}
\label{theoretical_analysis}
It has been shown that the certainty of LLM predictions can be quantified by issuing the same query multiple times and measuring the \textit{consistency} of the predictions~\cite{lyu2024calibrating}. Specifically, the confidence of ranking $s_a$ higher than $s_b$, $conf\{y = (s_a \succ s_b)\}$, is defined as $\frac{N(s_a \succ s_b)}{N_{query}(s_a, s_b)}$. where $N(s_a \succ s_b)$ denotes the number of times LLM ranks $s_a$ higher than $s_b$, and $N_{query}(s_a, s_b)$ denotes the total number of queries on $s_a$ and $s_b$. Confidence is a necessary condition to consider when evaluating LLM feedback quality, given that low confidence often causes considerable noise in feedback which manifests as incorrect rewards.
Based on the definition of feedback confidence, we derive an equivalent form of the RLHF loss based on ranking confidence and consistency as follows:

\begin{equation}
\begin{aligned}
    \mathcal{L} &= -\mathbb{E}_{(s_a, s_b, y) \sim \mathcal{D}} \bigg[ \mathbb{E}_{N_{query}} \Big[ \mathbb{I}\{y = (s_a \succ s_b)\} \\ &\log P_\psi[s_a \succ s_b] + \mathbb{I}\{y = (s_b \succ s_a)\} \\& \log P_\psi[s_b \succ s_a] \Big] \bigg]\\
    &= -\mathbb{E}_{(s_a, s_b, y) \sim \mathcal{D}} \bigg[ conf\{y = (s_a \succ s_b)\} \\ & \log (sigmoid(\sigma_\psi(s_a) - \sigma_\psi(s_b))) + \\& 
    conf\{y = (s_b \succ s_a)\} \\& \log (sigmoid(\sigma_\psi(s_b) - \sigma_\psi(s_a))) \bigg].
\end{aligned}
\label{eq:conf_rlhf}
\end{equation}
This loss function uses confidence-based weights to relate the scores between each state in ranked pairs.
From this derivation, we see the following.
1) A more confident ranking produces a larger score difference between the ranked states, i.e., the magnitude of the score difference is proportional to the confidence.
Formally, if $conf\{y = (s_a \succ s_b)\}$ > $conf\{y = (s_b \succ s_a)\}$ then $\sigma_\psi(s_a) - \sigma_\psi(s_b)$ > 0.
In the event the LLM is perfectly confident, $conf\{y = (s_a \succ s_b)\} = 1$, then the loss function will maximize $\sigma_\psi(s_a) - \sigma_\psi(s_b)$. 
2) As the confidence \textit{decreases}, then $|conf\{y = (s_a \succ s_b)\} - conf\{y = (s_b \succ s_a)\}|$ converges to $0$.
When the LLM is completely uncertain then $conf\{y = (s_a \succ s_b)\} = conf\{y = (s_b \succ s_a)\}$ and the loss function will minimize $|\sigma_\psi(s_a) - \sigma_\psi(s_b)|$, resulting in identical state scores such that $\sigma_\psi(s_a) = \sigma_\psi(s_b)$. The formal proof is given in Appendix \ref{formal_proof}.

\subsection{Potential-based Rewards for Learning with Noisy Feedback}

The above observations stemming from Eq.~\ref{eq:conf_rlhf} motivate the form of our proposed method.
Intuitively, when the LLM is completely uncertain when ranking $s_a$ and $s_b$ then the difference between their scores is zero.
This is ideal, as \textbf{when the LLM is unable to issue an accurate ranking then we would like it to issue an uninformative reward}, i.e., a reward of zero.
Our solution is to treat the state-score as a \textit{potential function}, defining the reward as the difference between the scores of successive state pairs:

\begin{equation}
\label{eq:rf}
    r(s_t) = \sigma_\psi(s_{t}) - \sigma_\psi(s_{t-1}).
\end{equation}

Thus, the more uncertain an LLM's ranking is, the less informative the reward is.
The potential in Eq.~\ref{eq:rf} is naturally shaped to a proper scale range, with positive rewards for actions that are beneficial and promising to
the given task goal and negative rewards for detrimental actions.
Large values correspond to more confident rankings, while small ones to less confident rankings.
As such, our approach is particularly well-suited to RLAIF with smaller, specialized models which are often necessary in resource-constrained environments.

There is an additional benefit to this formulation.
Prior works treat the state-score function as a reward function and directly issue rewards from it, which we call the ``direct-reward'' method.
This often leads to training instability as the rewards may have significant differences in scale,  
which need to be corrected via post-processing techniques such as normalization and thresholding as well as extrinsic per-step reward penalties.
However, the performance of post-processed direct rewards is highly sensitive to these hyper-parameters, as they are often task-specific.
Our potential difference formulation helps alleviate this issue as 1) uncertain states converge to the same score value so the impact of noisy rankings no longer needs to be mitigated through post-processing, and 2) per-step penalties can be discarded in favor of simple timestep-based discounting which are far less sensitive.

\subsection{Algorithm}

Our algorithm consists of the following four steps:
1) Randomly sample pairs of sequential states from the environment.
2) Query the LLM to rank states in each pair with respect to a natural language task description, e.g., ``Go to the green goal''. The prompt contains a language task description, environment description, and in-context learning examples \citep{wei2022chain} as context to generate preference labels for states in each pair. 
3) Train the state-score model $\sigma_\psi$ with the loss function in Eq. \ref{eq:loss}.
4) Train a reinforcement learning agent with feedback from the state-score model.

\section{Performance Analysis of Potential-Difference Rewards}

We evaluate our approach in commonly-used discrete (Grid World) and continuous (MuJoCo)~\citep{brockman2016openai} benchmark environments.
Throughout these experiments, we investigate the effectiveness of our potential-based reward function a) as compared to using the score as a reward directly in both single and multi-query settings; and b) its sensitivity to inconsistency in state rankings.

\subsection{Experiment Setup}
\label{sec:exp_setup}
\noindent\textbf{Grid World.} We examine three scenarios within Grid World \cite{swamy2024minimaximalist}: \textbf{NoLock}, \textbf{Lock}, and \textbf{MultiLock}. The layouts are shown in Fig.~\ref{fig:environment}. 
In each scenario, the agent (green triangle) must navigate to the target (green rectangle). 
There are one and two locked doors in the Lock and MultiLock variants, respectively, that block the agent's way to the goal. 
To unlock each door the agent must pick up the appropriate key and use it on the door. The agent, goal, and key positions are randomly initialized in every episode.

\begin{figure}[H]
\centering
  \includegraphics[width=0.9\columnwidth]{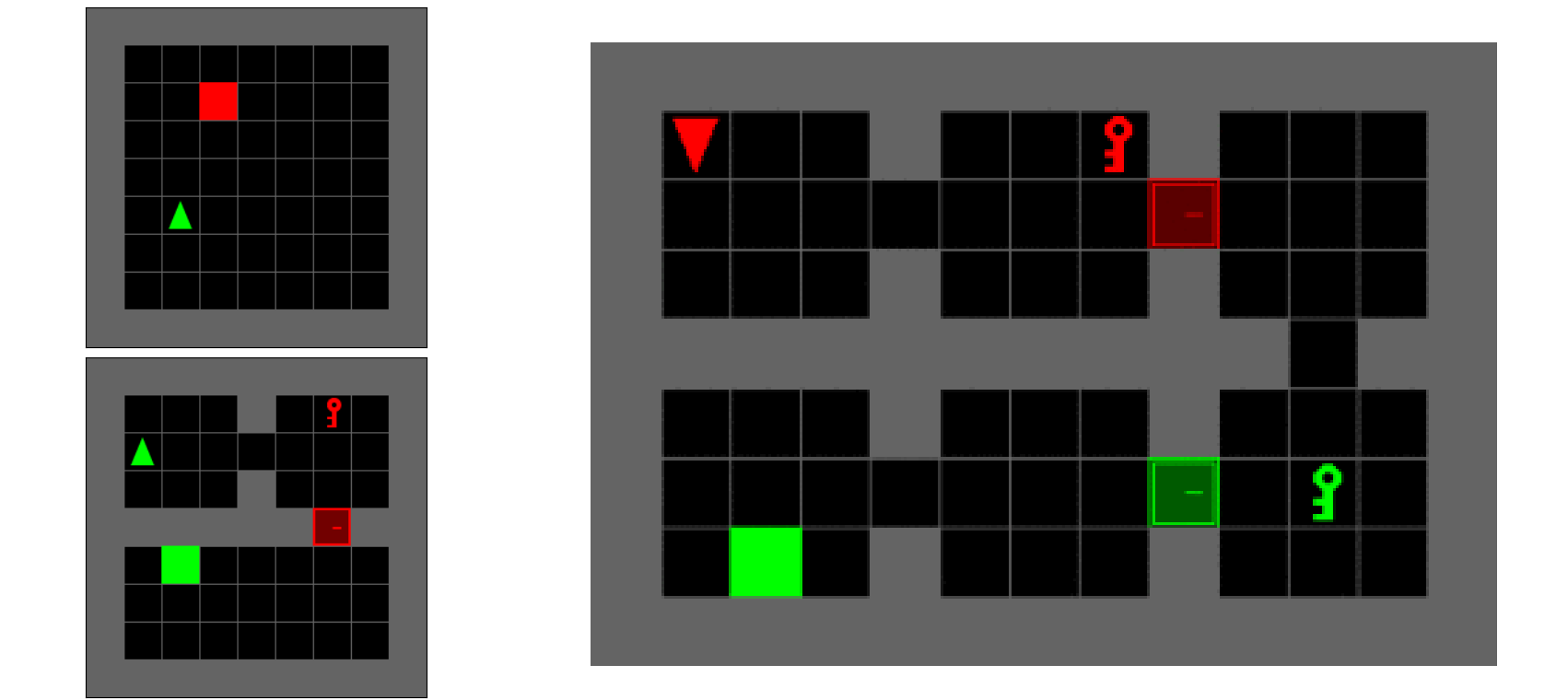} 
  \caption {Grid world environments with NoLock (upper-left), Lock (lower-left), and MultiLock (right) variants from left to right.}
  \label{fig:environment}
\end{figure}

\noindent\textbf{MuJoCo.} We examine a subset of robot control tasks selected from the simulated robotics benchmark MuJoCo~\cite{todorov2012mujoco}.
We choose 3 tasks with increasing degrees of complexity: \textbf{Inverted Pendulum}, \textbf{Reacher}, and \textbf{Hopper}.

For each of these six environments, we compare our approach with the following baselines:
\begin{itemize}
    \item {\bf Direct reward} directly utilizes the trained state-score functions' score as reward; i.e., $r(s) = \sigma_\psi(s)$. Following \citet{christiano2017deep}, the reward is normalized to zero-mean with a standard deviation of 1.
    \item {\bf Default reward} utilizes the vanilla RL objective of each environment with human-specified reward functions. 
    In grid world variants, the default reward is defined as $0$ for failure and $1-0.99n/n_{max}$ otherwise when the agent reaches the goal, where $n_{max}$ is the maximum time steps for each episode and $n$ is the step count until success. 
    For MuJoCo tasks, the default rewards are specified as those defined in OpenAI Gym \citep{brockman2016openai}. 
\end{itemize}

In each environment, we randomly sample pairs of sequential states from the environment with replacement in order to generate rankings for training the state-score model used by both potential difference and direct reward.
For single-query experiments, Grid World uses 2500, 3500, and 6000 samples for NoLock, Lock, and MultiLock respectively while MuJoCo uses 1000 samples for each environment.

Without loss of generality, we employ PPO as the underlying policy-training framework \citep{schulman2017proximal} and make the following assumptions: 
a) the environment is fully observable;
and b) the agent has no knowledge of the task before training, i.e., is randomly initialized.

\begin{figure*}[t]
\centering
  \includegraphics[width=0.9\linewidth]{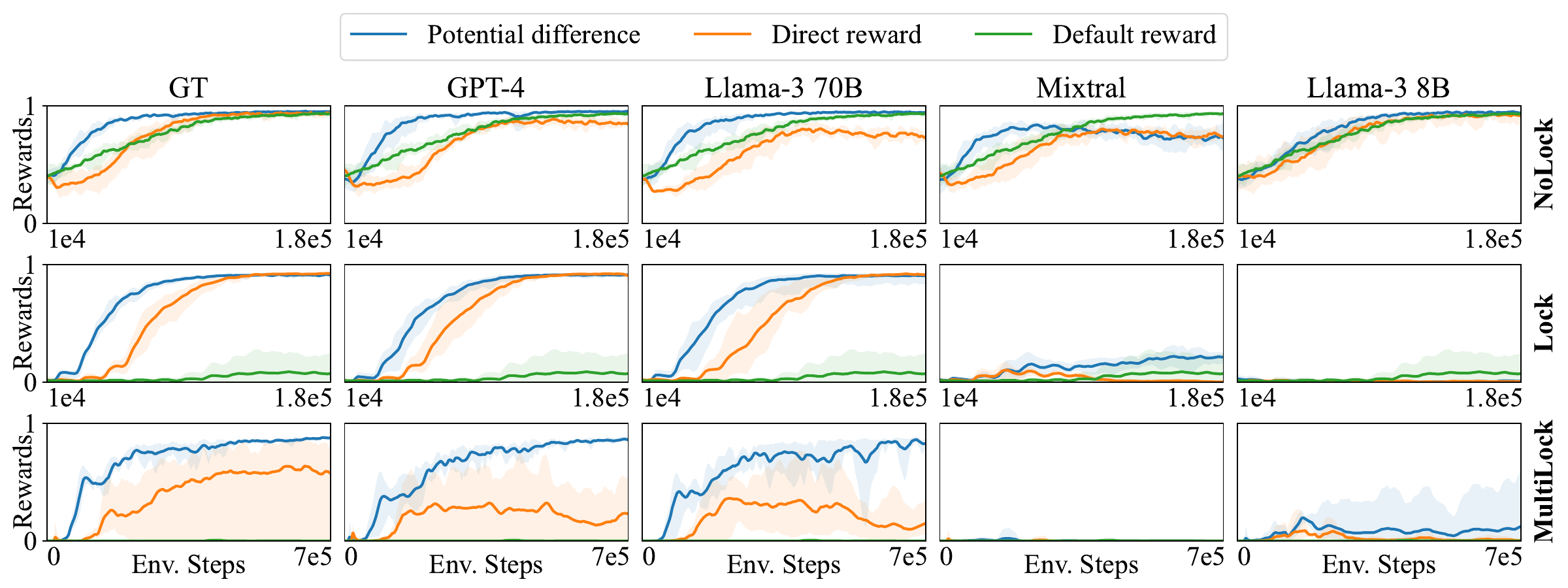} 
  \caption {The average learning curves with reward functions trained from single LLM queries in the Grid World environments over 5 random seeds, with the return variance visualized as shaded areas.}
    \label{fig:gridworld_single}
\end{figure*}
\begin{figure*}[t]
\centering
  \includegraphics[width=0.9\linewidth]{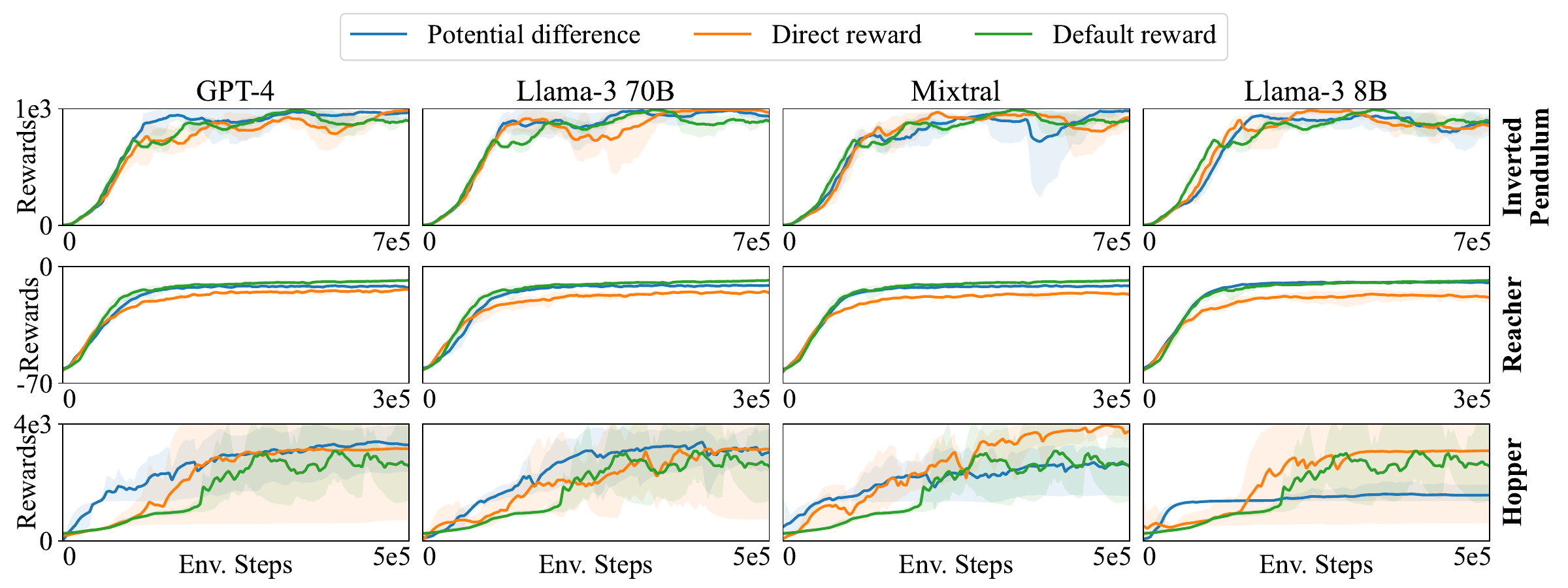} 
  \caption {The average learning curves with reward functions trained from single LLM queries in the MuJoCo environments over 5 random seeds, with the return variance visualized as shaded areas.}
    \label{fig:mujoco_single}
\end{figure*}

\subsection{LLM Ranking Performance}
\begin{table}[htbp]
    \setlength{\tabcolsep}{1.5pt}
    \footnotesize
    \centering
    {\renewcommand{\arraystretch}{1.2}
    
    \begin{tabular}{llccccc}
        \toprule
        Env. & Mthd. & \multicolumn{1}{c}{GT} & \multicolumn{1}{c}{Llama-3 8B} & \multicolumn{1}{c}{Mixtral} & \multicolumn{1}{c}{Llama-3 70B} & \multicolumn{1}{c}{GPT-4}
        \\ \hline
        \multirow{2}{*}{\rotatebox{90}{\parbox{2.7em}{\centering No\\Lock}}} \hspace{0.5em} & Rank  & 1.0 & 0.69 & 0.76 & 0.93 & 1.0 \\
        & Score  & 1.0 & 0.77 & 0.89 & 0.98 & 1.0 \\
        \hline
        \multirow{2}{*}{\rotatebox{90}{\parbox{2.7em}{\centering Lock}}} & Rank  & 1.0 & 0.54 & 0.65 & 0.89 & 0.98 \\
        & Score  & 1.0 & 0.55 & 0.74 & 0.97 & 0.98 \\
        \hline
        \multirow{2}{*}{\rotatebox{90}{\parbox{2.7em}{\centering Multi\\Lock}}} & Rank  & 1.0 & 0.58 & 0.60 & 0.90 & 0.99 \\
        & Score  & 1.0 & 0.66 & 0.66 & 0.96 & 0.99 \\
        \bottomrule
    \end{tabular}
    }
    \caption{Ranking accuracy for each LLM across 1000 state pairs sampled from each environment. Rank indicates the direct ranking performance of the LLMs and Score indicates the ranking performance of the trained score models. Given that the ground-truth ranking are only accessible in grid world environments, we only show the ranking correctness of LLMs and state-score models in these three environments.}
    \label{tab:llm variants}
\end{table}

We first quantify the performance of four different LLM models used in this work over each Grid World environment.
After sampling pairs of sequential states as discussed in Sec.~\ref{sec:exp_setup}, we measure the accuracy of a) the LLM's rankings and b) the resulting state-score model with respect to ground truth rankings.
The results, shown in Table~\ref{tab:llm variants}, provide us with an approximate ordering of LLM ranking performance, where GPT-4 $>$ Llama-3 70B $>$ Mixtral $>$ Llama-3 8B.

\begin{figure*}[t]
\centering
\vspace{-10pt}
  \includegraphics[width=0.9\linewidth]{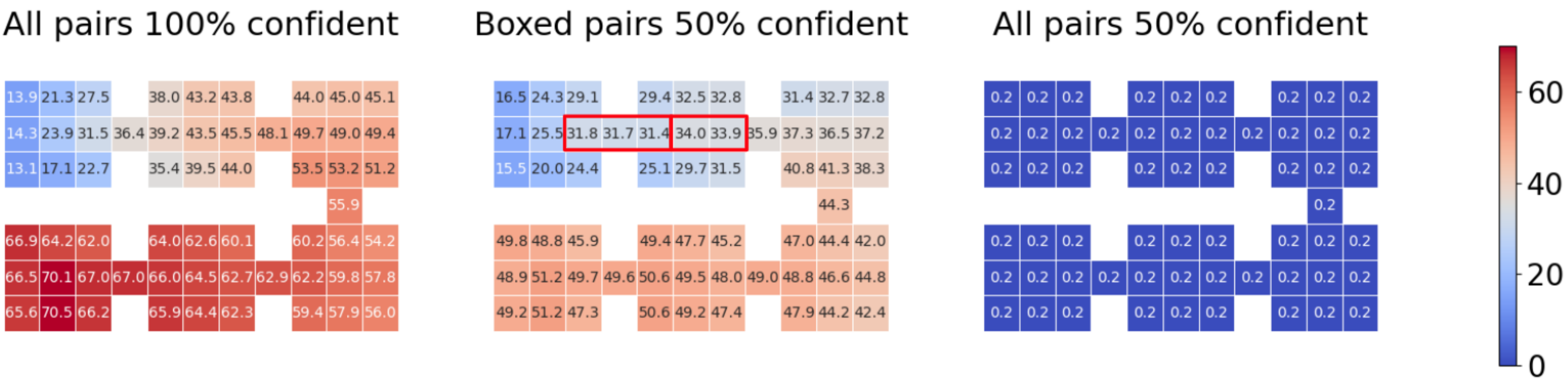} 
  \caption {The heat maps showing that feedback inconsistency pushes rewards towards 0. Each grid in the map shows the score of the state where the agent is at this grid. The first heat map shows state scores trained with 100\% confident rankings on all state pairs. The second heat map shows state scores trained with 100\% confident ranking on all state pairs except 50\% confident rankings on state pairs where the agent is in the red block. The third heat map shows state scores trained with 50\% rankings on all state pairs.}
    \label{fig:heatmaps}
\vspace{-10pt}
\end{figure*}
\subsection{Single-Query Evaluation}

We next examine how our approach performs compared to the standard direct reward approach commonly utilized in RLHF.
In each environment, we train our state score models with 4 LLMs: Mixtral~\citep{jiang2024mixtral}, GPT 4~\citep{achiam2023gpt}, and Llama-3 with 8B and 70B parameters ~\citep{touvron2023llama}. For Grid World environments, we add an additional baseline in which rankings are generated using a ground truth heuristic function (GT) which serves as an upper bound for our methods.

The state score models are trained by minimizing the loss in Eq. \ref{eq:loss}. 
Then they are employed to train 5 RL policies with random seeds and initializations for each method. 
As a common approach to avoid reward hacking, a constant step penalty  is applied to the produced rewards from both methods in all environments except for MuJoCo Reacher, which exploits a torque penalty as described in \citet{brockman2016openai}. 
The results, as well as the default reward performance, are shown in Fig.~\ref{fig:gridworld_single} and Fig.~\ref{fig:mujoco_single}.

In Grid World environments, our method (in blue) of potential difference-based reward outperforms the direct reward method in most cases. 
When using GT, GPT-4, or Llama-3 70B rankings, our method converges the fastest and yields the highest final reward. 
For the environments of Lock and MultiLock (bottom two rows), the tasks are more challenging when using the default reward; however, our method remains on par, or outperforms the baseline with respect to convergence speed and final reward. However, when using LLMs which generate noisy outputs (i.e., Mixtral and Llama-3 8B), all methods fail to converge in the Lock and MultiLock environments. 
In Sec.~\ref{sec:multi-query}, we detail our approach of using multiple queries, particularly for low-performing LLMs, to regain training performance using our potential-based reward function.

In MuJoCo environments, reported in Fig.~\ref{fig:mujoco_single}, our potential-based reward method slightly outperforms (particularly in Hopper with GPT-4 and Llama-3 70B) or is on par with our baseline methods. 
Exceptions to this trend can be seen with low-performing LLMs (e.g., Llama-3 8b).
Our method outperforms direct reward in Reacher and achieves a performance similar to the well-crafted default reward function, showing that potential-difference rewards are better. 
However, direct reward outperforms ours when using low-performing LLMs, particularly Mixtral and Llama-3 8B.
We attribute this to the challenge of designing appropriate prompts based on human intuition, i.e.,
we prompt LLMs to compare the hopper's speed in two consecutive 
states because the hopper should learn to move forward without falling down.
However, these LLMs then encourage moving faster instead of simply moving forward. 
While this prompt could lead to a good reward for high-performing LLMs, low-performing LLMs could not handle such situations, and we hypothesize that this leads to sub-optimal training results.

\begin{figure}[htbp]
    \includegraphics[width=1\columnwidth]{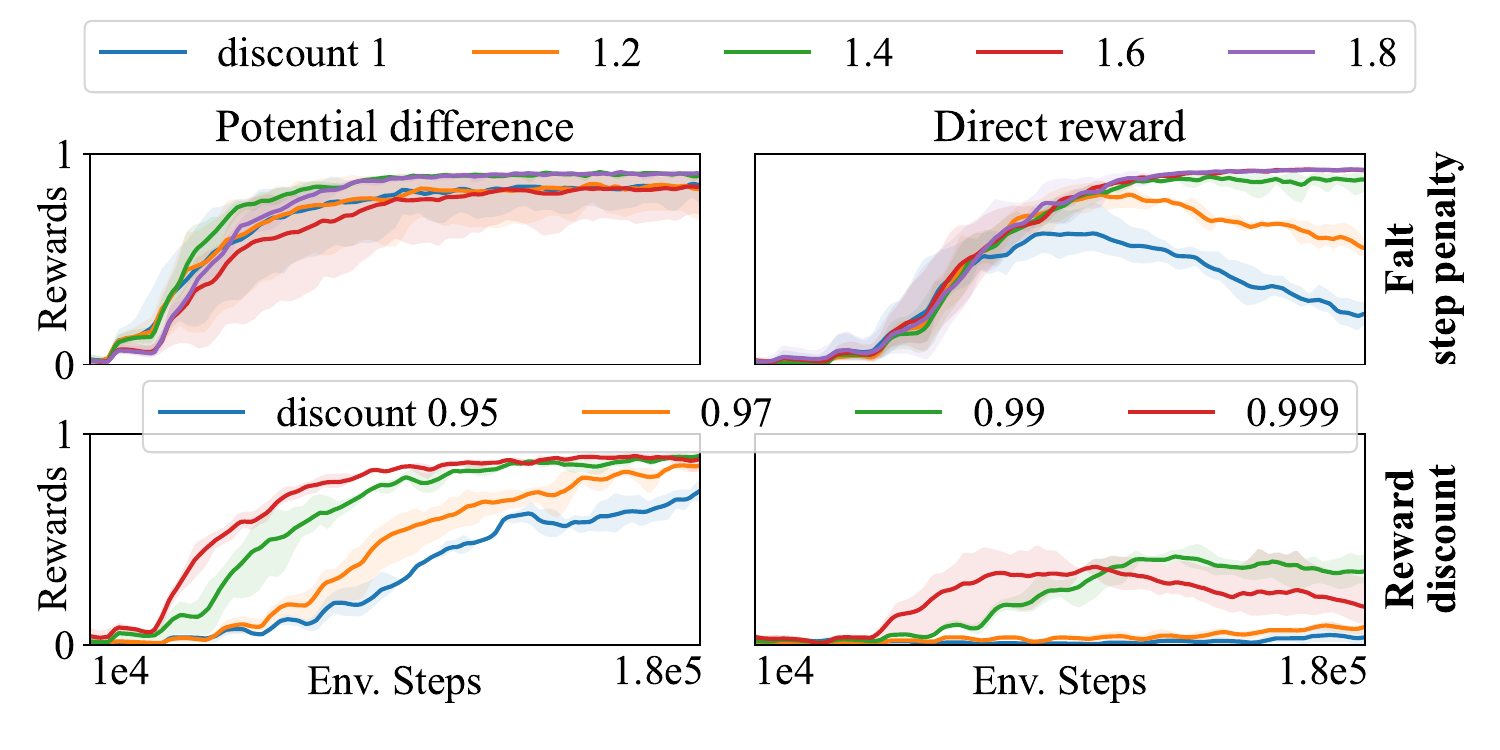}
    \caption{The average learning curves for rewards with multiple step penalties or discounts in the Grid World - Lock scenario, over 3 random seeds.}
    \label{fig:sensitivity_env2}
\vspace{-10pt}
\end{figure}

\subsubsection{Hyper-Parameter Sensitivity Analysis}

Since potential difference reward and direct reward suffer from reward hacking without post-processing, a step penalty is essential; however, choosing this value can be difficult. 
We conduct further experiments in the grid world Lock scenario to show that our method is less sensitive to step penalty than direct reward. 
Two penalty schemes with multiple parameters are tested: 1) \textbf{Flat Step Penalty}: A positive constant is subtracted at each time step. 2) \textbf{Reward Discount}: 
Reward for episode step $t$ is discounted by $\gamma_r^{t}$, where $\gamma_r < 1$ is a positive constant. 
We use the state score model trained from ground truth human heuristics for comparison. Each parameter is tested to train 3 RL policies with random seeds and initialization.
The results are shown in Fig. \ref{fig:sensitivity_env2}.

Our method shows robustness to the choice of the flat step penalty, as the curves of penalty variances are less divergent.
However, when using it as a direct reward, it can be seen that the performance is affected significantly with respect to the penalty, as many of them prevent the agent from converging.
The results also show that our method can perform well by picking a commonly-used discount factor such as 0.99, avoiding the burden of extensive hyperparameter tuning. 
However, using it as a direct reward requires further hyperparameter tuning.

\begin{figure*}[t]
\centering
  \includegraphics[width=1\linewidth]{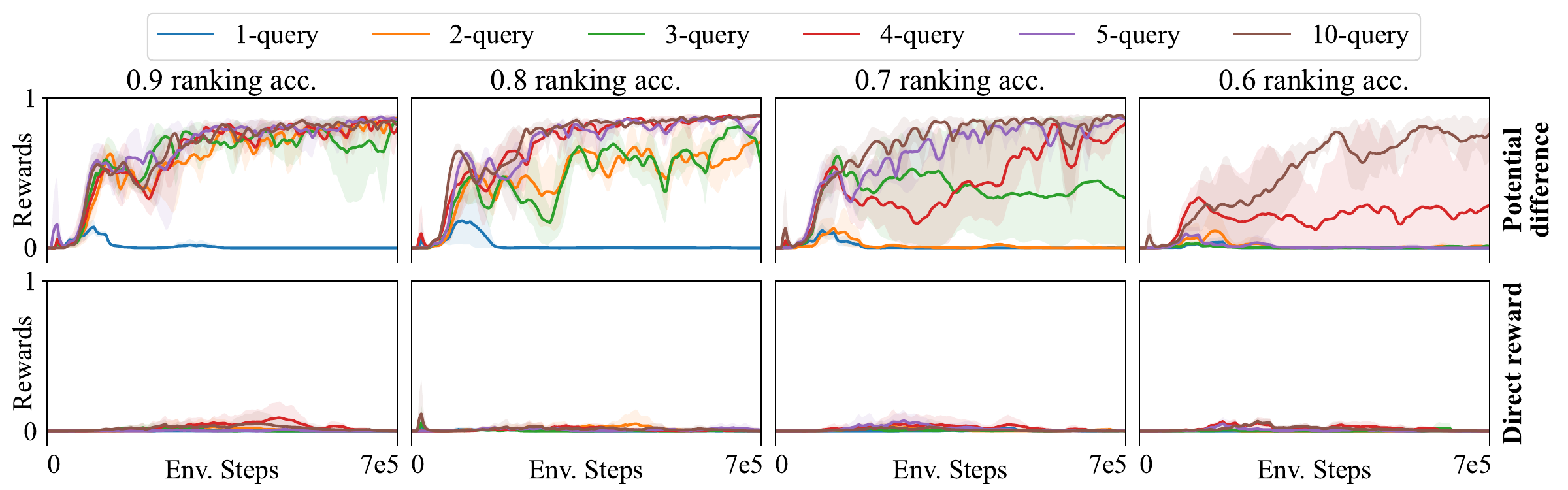} 
  \caption {The average learning curves with rewards trained from synthetic multi-query ranking datasets in the Grid World - MultiLock scenario, over 3 random seeds.}
    \label{fig:synthetic}
\end{figure*}
\begin{figure*}[t]
\centering
  \includegraphics[width=\linewidth]{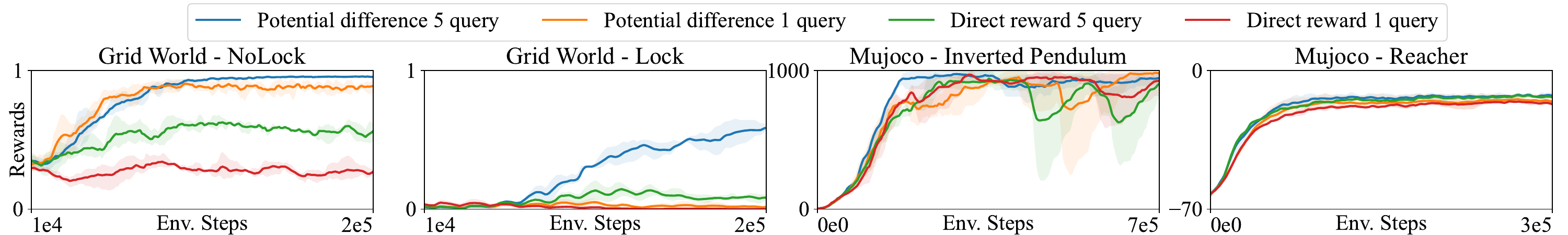} 
  \caption {The average-performance comparison of 5-query variation of potential-difference rewards and direct rewards with 1800, 2200, 1000, 1000 state pairs ranked with Mixtral, over 5 random seeds.}
    \label{fig:mixtral multi-query}
\end{figure*}
\subsection{Multi-Query Evaluation}
\label{sec:multi-query}
We introduce a multi-query approach that queries about ranking each state-transition case in the scoring-model training dataset a given number of times to address the rankings' inconsistency with lower-performing LLMs to push potential difference rewards toward zeros in the face of conflicting responses.
In Fig.~\ref{fig:heatmaps}, we illustrate the heat maps of state scores trained with datasets of distinct consistency degrees, demonstrated in the Grid World MultiLock environment.
Each grid in the heat maps records the score the scoring model assigns for an agent at that location.
The left sub-figure demonstrates the ideal case in which 100\% correct rankings are utilized to train the scoring model, demonstrating a smooth gradient 
from the start room (top left corner) to the final room (bottom left corner) roughly following the correct path.
However, if the scoring model is trained with 50\% confidence on all state pairs (right sub-figure in Fig.~\ref{fig:heatmaps}), the score in any state becomes equal as no adjacent states are ranked higher with high confidence. 
This demonstrates our method's ability to disregard states, and thus not provide rewards when LLM rankings are inconsistent across multiple queries. 
Finally, when the ranking results for a subset of states are inconsistent, yet consistent for all other locations (see Fig.~\ref{fig:heatmaps} center), 
the correct gradual change in score is maintained outside of the affected area. 
These results underline our method's capabilities with respect to the effects of pushing uncertain state scores toward zero while giving contrasting rewards to confident pairs, ultimately improving performance of our method with low-performing LLMs (see Sec.~\ref{sec:mq_eval}).

\subsubsection{Synthetic Ranking Evaluation}
\label{sec:mq_eval}
To test what ranking accuracy of datasets is needed for the LLM
with multi-query methods, and  how many queries are required, 
we synthesized training datasets with specific ranking accuracy from 60\% to 90\% and simulated query times from 1 to 10. 
State-score models trained with these datasets output rewards when training RL policies, and their performance is shown in Fig. \ref{fig:synthetic}. 
The specific ranking correctness rates are controlled by introducing random ranking errors into the ground-truth ranking datasets. 
This approach is repeated on several copies of the ground-truth datasets to simulate the multi-query ranking results. 

The result demonstrates that with more and more queries, the potential-difference reward gradually improves the training performance.
Two or more queries may achieve fast policy training converging towards optimal if using ranking datasets with high feedback consistency to train state-score models.
Notably, even for the
%if using 
datasets of only 60\% ranking accuracy, which is close to random guessing, potential-difference rewards trained with enough queries can still increase the average policy training returns from 0 to an almost optimal level with 10 queries. 
This indicates that with enough queries, even the datasets with low-ranking accuracy can be boosted to function like those with high accuracy.
This finding is consistent with our theoretic analysis and demonstrates considerable potential in mitigating significant ranking errors.

\subsubsection{LLM Ranking Evaluation}
Observing that Mixtral has the highest inconsistency in ranking states and thus has the largest potential for improvement, we evaluate the 5-query variations of potential-difference rewards and direct rewards with ranking results from Mixtral to verify our claims.
Different methods' RL policy training curves averaged over 5 random seeds are compared in Fig. \ref{fig:mixtral multi-query}. 
As hypothesized, the 5-query potential-difference rewards achieve faster policy training and result in the highest rewards in all experiments. 
The single-query potential-difference rewards also outperform the single-query direct rewards. 
The improvement is most significant in Grid World - Lock scenario.
\section{Conclusions}
In this work, we propose a simple method for incorporating noisy LLM feedback into reinforcement learning.
Our approach is based on learning a potential-based score function over repeated LLM-generated preference rankings which issues uninformative rewards for states in which the LLM is uncertain.
We show both theoretically and empirically that this results in a natural trade-off between reward sparsity and LLM confidence in a variety of discrete and continuous environments.

\section{Limitations}
Our current analysis is limited to relatively simple discrete and continuous environments so that we could perform a thorough empirical evaluation.
However, the consequence of this is that several of the LLMs, e.g., GPT-4, perform exceptionally well when ranking which leaves limited room for improvement.
This is especially notable in the MuJoCo environments where our potential difference approach results in insignificant changes to performance.
On the other hand, smaller-parameter models such as Mixtral exhibit worse performance and as such benefit more from our approach (Fig.~\ref{fig:mujoco_single}) which is in-line with our synthetic experiments (Fig.~\ref{fig:mixtral multi-query}).
In the future, we wish to explore more complex, realistic environments which induce similar ranking errors in a larger set of language models. Our method is theoretically compatible with visual and multimodal environments that possess richer state and action spaces and local observations, which can be ranked by LLMs or VLMs. Exploring these scenarios will be the focus of our future work.
A further limitation is that we currently assume that sequential state pairs can be randomly sampled from the environment.
While this is true in most simulated environments, this assumption is violated in others such as the real world.
In future work, we will explore iterative algorithms which alternate training the policy and sampling state pairs for ranking.

\section*{Acknowledgments}
This work is sponsored by Honda Research 58629.1.1012949, AFOSR FA9550-18-1-0251 and DARPA  FA8750-23-2-1015. The authors would also like to thank Dr. Woojun Kim for his assistance during discussion of this research.

\bibliography{custom}

\newpage
\appendix

\section{Theoretical Proof: Inconsistent rankings lead to uninformative rewards}
\label{formal_proof}

\begin{lemma}
    In the scope of RL based on LLM feedback, the confidence-based preference loss is equivalent to the standard preference loss used by state-score model training over multi-query ranking datasets.
\end{lemma}
\begin{proof}
    Given that the the confidence of ranking $s_a$ higher than $s_b$, $conf\{y = (s_a \succ s_b)\}$, is defined as $\frac{N(s_a \succ s_b)}{N_{query}(s_a, s_b)}$. where $N(s_a \succ s_b)$ denotes the number of times LLM ranks $s_a$ higher than $s_b$, and $N_{query}(s_a, s_b)$ denotes the total number of queries on $s_a$ and $s_b$.\\\\
    The standard preference loss over multiple-query ranking dataset $\mathcal{D}$ can be written as
    \begin{equation}
    \begin{aligned}
        &\mathcal{L}_\mathcal{D} = -\mathbb{E}_{(s_a, s_b, y) \sim \mathcal{D}} \bigg[ \mathbb{E}_{N_{query}} \Big[ \mathbb{I}\{y = (s_a \succ s_b)\} \\& \log (sigmoid(\sigma_\psi(s_a) - \sigma_\psi(s_b))) + \\& \mathbb{I}\{y = (s_b \succ s_a)\} \\& \log (sigmoid(\sigma_\psi(s_b) - \sigma_\psi(s_a))) \Big] \bigg]\\
        &= -\mathbb{E}_{(s_a, s_b, y) \sim \mathcal{D}} \bigg[ \\& conf\{y = (s_a \succ s_b)\} \\& \log (sigmoid(\sigma_\psi(s_a) - \sigma_\psi(s_b))) + \\& 
        conf\{y = (s_b \succ s_a)\} \\& \log (sigmoid(\sigma_\psi(s_b) - \sigma_\psi(s_a))) \bigg].
    \end{aligned}
    \label{eq:conf_rlhf}
    \end{equation}
\end{proof}

\begin{theorem}
 As inconsistency of a ranking over two states increases, the scores of these two states converge to the same value.  
\end{theorem}
\begin{proof}
    Based on Equation \ref{eq:conf_rlhf},
    \begin{equation}
    \begin{aligned}
    \mathcal{L}_\mathcal{D} &= -\mathbb{E}_{(s_a, s_b, y)\sim \mathcal{D}} \bigg[ conf\{y = (s_a \succ s_b)\} \\& \log (sigmoid(\sigma_\psi(s_a) - \sigma_\psi(s_b))) + \\& (1 - conf\{y = (s_a \succ s_b)\}) \\& \log (1 - sigmoid(\sigma_\psi(s_a) - \sigma_\psi(s_b))) \bigg].
    \end{aligned}
    \end{equation}

    Take an arbitrary state pair $(s_0, s_1)$ from $\mathcal{D}$. As inconsistency of the ranking over $s_0$ and $s_1$ increases, $conf\{y = (s_0 \succ s_1)\} \to 0.5$. Denote $sigmoid(\sigma_\psi(s_0) - \sigma_\psi(s_1))$ as $p_{0,1}$, $\mathcal{L}$ over other states in $\mathcal{D}$ without $s_0, s_1$ as $\mathcal{L}_{\mathcal{D} \setminus \{s_0, s_1\}}$,
    \begin{equation}
    \begin{aligned}
    &\lim_{conf\{y = (s_0 \succ s_1)\} \to 0.5}\mathcal{L}_\mathcal{D} \\&\to -\frac{1}{2|\mathcal{D}|} \log (p_{0,1}(1 - p_{0,1})) + \mathcal{L}_{\mathcal{D} \setminus \{s_0, s_1\}} \\& \ge \frac{1}{|\mathcal{D}|} \log 2 + \mathcal{L}_{\mathcal{D} \setminus \{s_0, s_1\}}.
    \end{aligned}
    \end{equation}\\ If and only if $p_{0,1} \to \frac{1}{2}$, where $\sigma_\psi(s_0) - \sigma_\psi(s_1) \to 0$, 
    $\lim_{conf\{y = (s_0 \succ s_1)\} \to 0.5}\mathcal{L}_\mathcal{D} \to \frac{1}{|\mathcal{D}|} \log 2 + \mathcal{L}_{\mathcal{D} \setminus \{s_0, s_1\}}$, reaching the lower bound. \\
    
    Therefore, when training the state-score model with this loss $\mathcal{L}$, the scores of any two states whose ranking confidence is close to 50\% will be pushed to the same value.
\end{proof}

\section{Scoring-Model Training Datasets}
To train the scoring model, we randomly sample sequential state pairs and collect LLM ranking results on them, assembling all the data into a training dataset. The training data for all six environments can be accessed here: \href{https://drive.google.com/drive/folders/1odybuKPxqLzz_Y9Blw09V3TbXU4ELFjX?usp=drive_link}{Scoring-Model Training Data}. The details of collection process are as follows.

\subsection{LLM Preference Generating Process}
The LLM does pairwise state ranking in this work. We follow the methodology described in \cite{lee2023rlaif}, where the LLM prompt consists of four parts:

\begin{enumerate}
    \item \textbf{Preamble:} A description of the environment, task, and ranking criteria.
    \item \textbf{Few-shot exemplar:} Pairwise state-ranking example, showcasing the chain of thought on ranking according to given environment conditions and state evaluation criteria.
    \item \textbf{Sample to annotate:} The pair of specific states a and b, described with natural language.
    \item \textbf{Ending:} Ending text to prompt a preferred response as ranking.
\end{enumerate}

In the generated response, the LLM determines the ranking based on the specified criteria between two sequential states and outputs either ‘Yes’ (a is ranked higher than b) or ‘No’ (b is ranked higher than a). The following is an example prompt for the Inverted Pendulum environment:

\begin{tcolorbox}[colback=white, colframe=black, width=\columnwidth, arc=0mm, auto outer arc,
                  boxrule=0.5mm, fonttitle=\bfseries, title=Preamble]
                  
You are in an environment, which involves a cart that can move linearly, with a pole fixed on it at one end and having another end free. The cart can be pushed left or right, and the goal is to balance the pole on the top of the cart by applying forces on the cart.

\end{tcolorbox}

\begin{tcolorbox}[colback=white, colframe=black, width=\columnwidth, arc=0mm, auto outer arc,
                  boxrule=0.5mm, fonttitle=\bfseries, title=Few-shot exemplar]
Q:\\
State[a]:\\
The pole leans to the left by 0.1 radians with a leftward velocity.
\\

State[b]:\\
The pole stands vertically with a rightward velocity.
\\

Is the transition from State[a] to State[b] a good transition?
\\

A:\\
Yes, the pole currently stands vertically, so it has been balanced. 
Therefore, the answer is yes.
\end{tcolorbox}

\begin{tcolorbox}[colback=white, colframe=black, width=\columnwidth, arc=0mm, auto outer arc,
                  boxrule=0.5mm, fonttitle=\bfseries, title=Sample to annotate]
Q:\\
State[0]:\\
The pole leans to the right by 0.6 radians with a leftward velocity.
\\

State[1]:\\
The pole leans to the right by 0.1 radians with a rightward velocity.
\end{tcolorbox}

\begin{tcolorbox}[colback=white, colframe=black, width=\columnwidth, arc=0mm, auto outer arc,
                  boxrule=0.5mm, fonttitle=\bfseries, title=Ending]
Is the transition from State[0] to State[1] a good transition?
\end{tcolorbox}

Prompts for other environments can be accessed here: \href{https://drive.google.com/drive/folders/1-2G4X_KzcPzswmmEM21Wl892_i5EkcC-?usp=sharing}{Prompts for all 6 environments}. To maintain consistent settings across all experiments and eliminate the influence of irrelevant variables, we use the same prompts for all LLM models.

\section{Experiment Details for Reproducibility}

\subsection{Model Architecture} 

{\bf Grid World} The policy model for all scenarios contains separate actor and critic networks, both with 3 convolutional layers followed by 1 fully connected layer mapping the flattened vector to the output vector. The convolutional layers consist of 16 $2\times2$ filters, followed by $2\times2$ pooling, then 32  $2\times2$ filters, and finally 64  $2\times2$ filters. Scoring model architectures for each scenario are shown in Table \ref{tab:rm_grid}.

\begin{table}[htbp]
\setlength{\tabcolsep}{5pt}
\footnotesize
\centering
{\renewcommand{\arraystretch}{1.2}
\begin{tabular}{lccc}
\toprule
 & NoLock & Lock & MultiLock \\ \hline
Conv & \multicolumn{2}{c}{\begin{tabular}[c]{@{}c@{}}conv 16, (2,2),\\ pool (2,2),\\ conv 32, (2,2),\\ conv 64, (2,2),\\ conv 128, (2,2)\end{tabular}} & \begin{tabular}[c]{@{}c@{}}conv 16, (2,2),\\ pool (2,2),\\ conv 32, (2,2), \\ conv 64, (2,2) \\ \end{tabular} \\ \hline
FC hidden & \multicolumn{2}{c}{\begin{tabular}[c]{@{}c@{}} 256, 128, \\ 64, 32, 16 \\ \end{tabular}} & \begin{tabular}[c]{@{}c@{}}512, 256, \\128, 64, 16 \\ \end{tabular} \\
\bottomrule
\end{tabular}
}
\caption{Scoring model architecture for Grid World scenarios. The number of output channels and kernal size is given for each convolutional layer. The number of nodes for each fully connected hidden layer are given.}
\label{tab:rm_grid}
\end{table}

{\bf MuJoCo} Policy model follows \citet{schulman2017proximal}, where both actor and critic networks have a fully connected network using a hidden layer with 64 nodes. Distinct scoring model architectures are used in each scenario, shown in Table \ref{tab:rm_mujoco}.
\begin{table}[htbp]
\setlength{\tabcolsep}{5pt}
\footnotesize
\centering
{\renewcommand{\arraystretch}{1.2}
\begin{tabular}{lccc}
\toprule
 & Reacher & Inverted Pendulum & Hopper \\ \hline
FC hidden & 128 & 128 & 128, 64 \\ \bottomrule
\end{tabular}
}
\caption{Scoring model architecture for MuJoCo scenarios. The number of nodes for each fully connected hidden layer are given.}
\label{tab:rm_mujoco}
\end{table}

\subsection{Hyperparameters}
The hyperparameters of training scoring models and PPO policies were tuned manually. The details are recorded in Table \ref{tab:ppo_hp_grid}, \ref{tab:rm_hp_grid}, \ref{ppo_hp_mujoco}, \ref{tab:rm_hp_mujoco}.

\begin{table}[htbp]
\setlength{\tabcolsep}{5pt}
\footnotesize
\centering
{\renewcommand{\arraystretch}{1.2}
\begin{tabular}{cccc}
\toprule
Hyperparameter & NoLock & Lock & MultiLock \\ \hline
Learning Rate & \begin{tabular}[c]{@{}c@{}}{[}0.0028, 150000{]}\\ {[}0.00000001, 180000{]}\end{tabular} & \begin{tabular}[c]{@{}c@{}}{[}0.008, 100000{]}\\ {[}0.001, 127000{]},\\ {[}0.000005, 180000{]}\end{tabular} & 0.0001 \\
Batch Size & 1024 & 2048 &2048 \\
Num. SGD Epochs & 4 & 4 & 4 \\
Minibatch Size & 128 & 128 & 128 \\
Clipping Prameter & 0.2 & 0.2 & 0.2 \\
VF Clip Parameter & 10.0 & 10.0 & 10.0 \\
VF Coeff. & 0.5 & 0.5 & 0.5 \\
KL Coeff. & 0.5 & 0.5 & 0.5 \\
Entropy Coeff. & 0.01 & 0.01 & 0.01 \\
GAE & 0.8 &  0.8 & 0.8 \\
Discount & 0.99 & 0.99 & 0.99 \\
\bottomrule
\end{tabular}
}
\caption{PPO hyperparameters for Grid World scenarios.}
\label{tab:ppo_hp_grid}
\end{table}

\begin{table}[htbp]
\setlength{\tabcolsep}{4pt}
\footnotesize
\centering
{\renewcommand{\arraystretch}{1.2}
    \begin{tabular}{cccc}
    \toprule
Hyperparameter & NoLock & Lock & MultiLock \\ \hline
SM. LR. Schedule & \begin{tabular}[c]{@{}c@{}}{[}0.004, 17{]}\\ {[}0.0001, 45{]},\\ {[}0.0000001, 250{]}\end{tabular} & \begin{tabular}[c]{@{}c@{}}{[}0.000000001, 20{]}\\ {[}0.00004, 45{]},\\ {[}0.0000001, 250{]}\end{tabular} & \begin{tabular}[c]{@{}c@{}}{[}0.000004, 20{]}, \\ {[}0.000004, 100{]}, \\ {[}0.0000001, 250{]}\end{tabular} \\

SM. Batch Size & 32 & 16 & 64 \\

SM. Epochs & 200 & 120 & 250 \\ \bottomrule
\end{tabular}
}
    \caption{Scoring model training hyperparameters for Grid World scenarios. Learning rate schedule is presented as the learning rate value along with the corresponding final epoch it is applied to.}
    \label{tab:rm_hp_grid}
\end{table}

\begin{table}[htbp]
\setlength{\tabcolsep}{5pt}
\footnotesize
\centering
{\renewcommand{\arraystretch}{1.2}
\begin{tabular}{cc}
\toprule
Hyperparameter & Value \\ \hline
Learning Rate & 0.0003 \\
Batch Size & 2048 \\
Num. SGD Epochs & 10 \\
Minibatch Size & 64 \\
Clipping Prameter & 0.2 \\
VF Clip Parameter & 10 \\
VF Coeff. & 1 \\
KL Coeff. & 0.2 \\
Entropy Coeff. & 0 \\
GAE & 0.95 \\
Discount & 0.99 \\
\bottomrule
\end{tabular}
}
\caption{PPO hyperparameters for all 3 MuJoCo environments}
\label{ppo_hp_mujoco}
\end{table}

\begin{table}[htbp]
\setlength{\tabcolsep}{4pt}
\footnotesize
\centering
{\renewcommand{\arraystretch}{1.2}
    \begin{tabular}{cccc}
    \toprule
Hyperparameter & Pendulum & Reacher & Hopper \\ \hline
SM. LR. Schedule & \begin{tabular}[c]{@{}c@{}}{[}0.000004, 20{]}\\ {[}0.00002, 40{]},\\ {[}0.000008, 50{]}\end{tabular} & \begin{tabular}[c]{@{}c@{}}{[}0.000004, 20{]}\\ {[}0.00002, 120{]},\\ {[}0.000008, 300{]}\end{tabular} & \begin{tabular}[c]{@{}c@{}}{[}0.000004, 50{]}, \\ {[}0.00002, 70{]}\end{tabular} \\

SM. Batch Size & 16 & 16 & 16 \\

SM. Epochs & 50 & 300 & 70 \\ \bottomrule
\end{tabular}
}
    \caption{Scoring model training hyperparameters for Mujoco scenarios. Learning rate schedule is presented as the learning rate value along with the corresponding final epoch it is applied to.}
    \label{tab:rm_hp_mujoco}
\end{table}

\end{document}